\documentclass[10pt,twocolumn,letterpaper]{article}

 \usepackage{cvpr}              % To produce the CAMERA-READY version
\definecolor{cvprblue}{rgb}{0.21,0.49,0.74}
\usepackage[pagebackref,breaklinks,colorlinks,allcolors=cvprblue]{hyperref}
\usepackage{amsthm}  % 核心宏包：用于定义定理/假设环境
\usepackage{multirow}
\usepackage{graphicx}
\usepackage{pifont}
\usepackage{makecell} % 提供\makecell命令
\usepackage{caption}
\usepackage[accsupp]{axessibility}  % Improves PDF readability for those with disabilities.

\usepackage[numbers]{natbib}  % for pass arXiv

\newtheorem{hypothesis}{Hypothesis}         % 全局连续编号（如 H1, H2）

\theoremstyle{plain}  % 定理用plain样式（内容斜体，突出逻辑陈述）
\newtheorem{theorem}{Theorem}         % 全局连续编号（推荐短文档）

\newtheorem{lemma}{Lemma}

\newcommand{\smileemoji}{\includegraphics[height=0.8em]{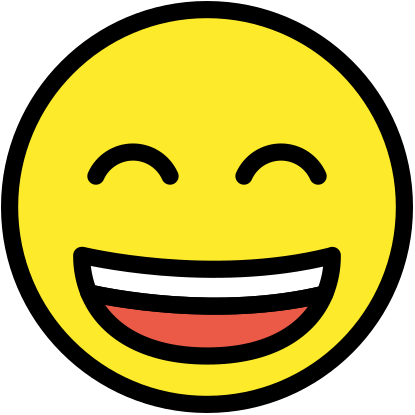}}
\newcommand{\tencent}{\includegraphics[height=0.7em]{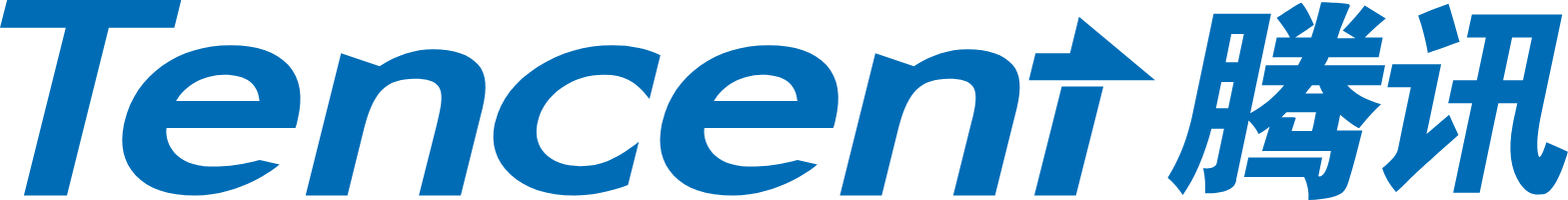}}
\newcommand{\bytedance}{\includegraphics[height=1em]{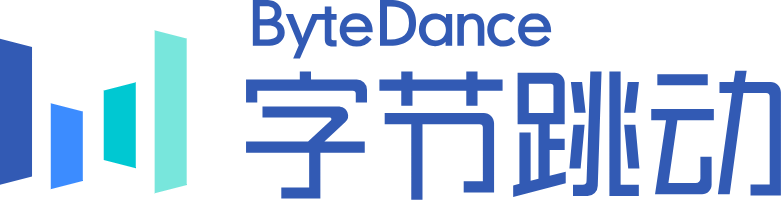}}
\newcommand{\antgroup}{\includegraphics[height=1em]{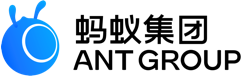}}
\newcommand{\baidu}{\includegraphics[height=1em]{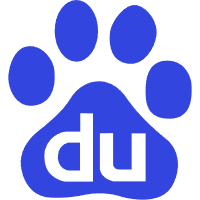}}
\newcommand{\zjdx}{\includegraphics[height=1em]{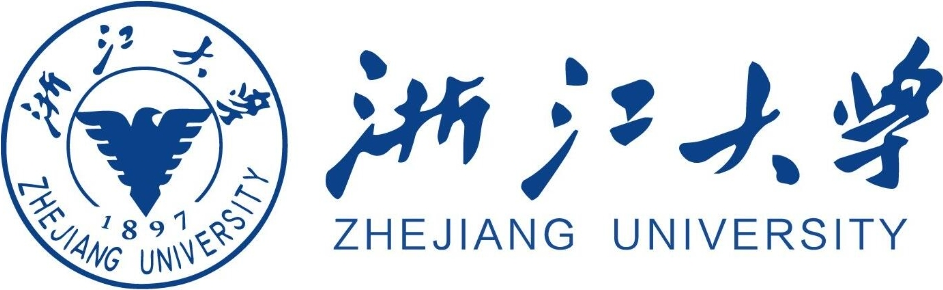}}
\newcommand{\cuhk}{\includegraphics[height=1em]{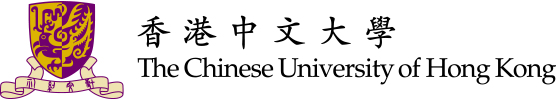}}
\newcommand{\bjjt}{\includegraphics[height=1em]{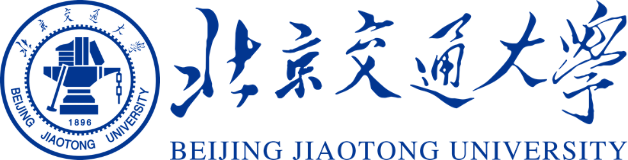}}
\newcommand{\fudan}{\includegraphics[height=1em]{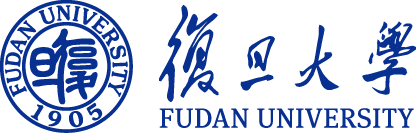}}
\newcommand{\nanda}{\includegraphics[height=1em]{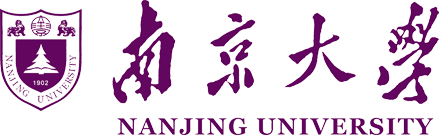}}
\newcommand{\mystar}{\includegraphics[height=0.8em]{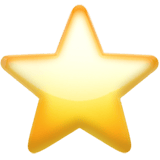}}

\def\paperID{***} % *** Enter the Paper ID here
\def\confName{CVPR}
\def\confYear{2026}

\title{IP-Adapter Is All You Need: Towards Fine-Tuning-Free Diffusion-Based Talking Face Generation}

\author{Hao Wu$^1$ \quad
		Xiangyang Luo$^1$ \quad
		Hao Wang$^2$ \quad
		Jiawei Zhang$^3$ \quad
		Yi Zhang$^{1,}$\footnotemark[2] \quad
		Jinwei Wang$^{4,}$\footnotemark[2] \\
		\textsuperscript{1}Information Engineering University \quad
		\textsuperscript{2}Huai'an University \\
		\textsuperscript{3}Chongqing University of Post and Telecommunications \quad
		\textsuperscript{4}Nankai University
}

\begin{document}
%\maketitle
\twocolumn[{
	\maketitle
	\begin{center}
		\captionsetup{type=figure}
%		\vspace{-2em}
		\includegraphics[width=0.7\textwidth]{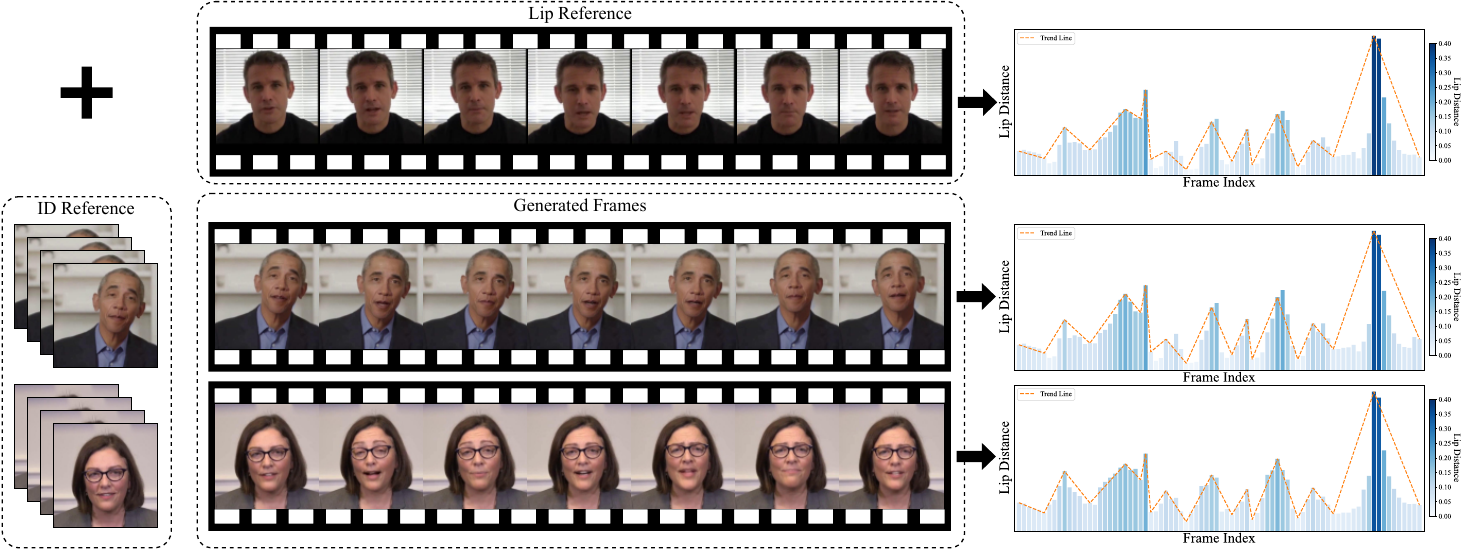}
%		\vspace{-2em}
		\captionof{figure}{Visualization of videos generated by the proposed fine-tuning-free diffusion-based framework, FreeTalkDiff. The framework produces realistic, high-definition, lip-synced talking videos from lip and ID references. The accompanying lip-distance bar chart quantitatively illustrates the close alignment between the generated and reference lip motions, highlighting precise lip-control capability.}
	\end{center}
}]

\renewcommand{\thefootnote}{\fnsymbol{footnote}}
\footnotetext[2]{Corresponding authors. E-mail: tzyy4001@sina.com, wangjinwei@nankai.edu.cn} 
\begin{abstract}
With the rapid advancement of diffusion models, talking face generation has made remarkable progress. However, existing diffusion-based methods still require task-specific fine-tuning and large-scale audiovisual datasets, resulting in high computational costs that hinder scalability and accessibility of diffusion-based approaches across the research community. To address this, we propose a fine-tuning-free paradigm that directly performs talking face generation using the pretrained weights of Stable Diffusion and IP-Adapter. This backbone leverages the visual embedding capability of IP-Adapter to mine lip-related semantics from the pretrained Stable Diffusion. To address the challenges of identity drift, synchronization errors, and temporal instability, we also design three trainable-parameter-free components: 1) the Structurist, which explicitly disentangles and reassembles lip and appearance features to mitigate identity drift and appearance distortion; 2) the Structure Controller, which adaptively refines embeddings based on quasi-monotonic motion trends for precise lip synchronization; and 3) the Noise Sensor, which introduces a Gaussian prior to detect and suppress flicker and jitter artifacts and enhance temporal consistency. Experimental results show that our method outperforms existing SOTA approaches in both lip-sync accuracy (at least 0.16 gain in PCLD) and visual fidelity (at least 0.7 improvement in FID), establishing a novel fine-tuning-free diffusion framework for talking face generation. The code is available at \url{https://github.com/tlemangen/FreeTalkDiff}.
\end{abstract}    
\section{Introduction}
\label{sec:intro}

\begin{table*}[t]
	\centering
	\caption{Information for diffusion model-based methods. ``\#Params" indicates the number of parameters. ``—" indicates no available data. ``\mystar" denotes the relative fine-tuning cost through a qualitative summary; more stars indicate higher cost.}
	\resizebox{\textwidth}{!}{%
		\begin{tabular}{@{}c|c|c|c|c|c|c@{}}
			\toprule
			\textbf{Method} & \textbf{Affiliations} & \textbf{Trainable \#Params } & \textbf{Train Sets} & \textbf{GPUs for Fine-tuning} & \textbf{Training Steps} & \textbf{Fine-tuning Cost} \\ 
			\cmidrule{1-7}
			AniPortrait(arXiv'24)~\cite{wei2024aniportrait} & \tencent & 2.59B & HDTF~\cite{zhang2021flow}, VFHQ~\cite{xie2022vfhq}, CelebV-HQ~\cite{zhu2022celebv} & 4 NVIDIA A100 & — & \mystar\mystar\mystar\mystar \\
			Loopy(arXiv'24)~\cite{jiang2024loopy} & \bytedance, \zjdx & — & Self-build, HDTF~\cite{zhang2021flow} & 24 NVIDIA A100 & — & \mystar\mystar\mystar\mystar\mystar \\
			MuseTalk(arXiv'24)~\cite{zhang2024musetalk} & \tencent, \cuhk & 0.85B & HDTF~\cite{zhang2021flow} & 2 NVIDIA H20 & 300k & \mystar\mystar\mystar \\
			LatentSync(arXiv'24)~\cite{li2024latentsync} & \bytedance, \bjjt & 1.64B & VoxCeleb2~\cite{chung2018voxceleb2}, HDTF~\cite{zhang2021flow} & — & — & \mystar\mystar\mystar \\
			EchoMimic(AAAI'25)~\cite{chen2025echomimic} & \antgroup & 2.15B & Self-build, HDTF~\cite{zhang2021flow}, CelebV-HQ~\cite{zhu2022celebv} & 8 NVIDIA A100 & 60k & \mystar\mystar\mystar\mystar \\
			Hallo2(ICLR'25)~\cite{cui2025hallo} & \fudan, \baidu, \nanda & 2.39B & Self-build, HDTF~\cite{zhang2021flow}, CelebV-HQ~\cite{zhu2022celebv} & 8 NVIDIA A100 & 113k & \mystar\mystar\mystar\mystar \\
			Sonic(CVPR'25)~\cite{ji2025sonic} & \tencent, \zjdx & 1.65B & VFHQ~\cite{xie2022vfhq}, CelebV-Text~\cite{yu2023celebv}, VoxCeleb2~\cite{chung2018voxceleb2} & — & — & \mystar\mystar\mystar \\
			\cmidrule{1-7}
			\textbf{Ours} & — & \textbf{0.00}~\smileemoji & \textbf{Not Required}~\smileemoji & \textbf{Not Required}~\smileemoji & \textbf{Not Required}~\smileemoji & \textbf{Not Required}~\smileemoji\\ 
			\bottomrule
		\end{tabular}%
	}
	\label{tab:diffusion_method_info}
\end{table*}

In recent years, talking face generation, which synthesizes talking videos of a target identity, has attracted widespread attention due to its vast application potential in areas such as entertainment~\cite{zhou2020makelttalk,ji2022eamm} and digital humans~\cite{meng2025echomimicv2,wang2025fantasytalking}. With the rapid advancement of diffusion models~\cite{ho2020denoising,song2020denoising,rombach2022high,blattmann2023stable}, their performance in image and video generation has significantly surpassed that of traditional generative adversarial networks~\cite{goodfellow2014generative,mirza2014conditional,radford2015unsupervised,goodfellow2020generative} and autoregressive models~\cite{hochreiter1997long,sak2014long,chung2014empirical}. In particular, diffusion models demonstrate unique advantages in generation quality, and multimodal conditional control, making them a major driving force in advancing talking face generation.

However, existing diffusion-based talking face generation methods require task-specific fine-tuning of diffusion models on large-scale audiovisual datasets. Given the massive parameter size and complex optimization of diffusion models, these methods often demand multiple high-performance GPUs and tens to hundreds of thousands of training steps, resulting in substantial computational and time costs. As shown in \cref{tab:diffusion_method_info}, existing diffusion-based methods typically involve fine-tuning models with billions of parameters. Representative examples include EchoMimic~\cite{chen2025echomimic} and Hallo2~\cite{cui2025hallo} were trained for 60k and 113k steps on eight A100, respectively, while Loopy~\cite{jiang2024loopy} utilized as many as twenty-four A100 GPUs. This heavy dependence on computational resources and prolonged optimization greatly limits the scalability and accessibility of diffusion-based approaches across the research community.

To overcome the aforementioned bottleneck, this paper explores a novel framework: \textbf{performing talking face generation directly with a pretrained diffusion model, without any task-specific fine-tuning.} This approach eliminates the need for additional training, thereby significantly reducing computational and time costs. However, this framework introduces a key challenge: \textbf{\textit{how to mine lip-related knowledge from the vast knowledge of the pretrained diffusion model, thereby enabling controllable lip-sync generation?}}

Our solution is motivated by an in-depth observation of pretrained model behaviors. We identify that the IP-Adapter~\cite{ye2023ip}, when paired with Stable Diffusion (SD)~\cite{rombach2022high}, naturally emphasizes lip-related features in facial representations, as the mouth region typically carries the main cues of the face. Therefore, we adopt a fine-tuning-free, controllable ``SD + IP-Adapter" combination as our backbone for talking face generation. Building upon this backbone, we further propose three novel components with no trainable parameters: 1) \textbf{Structurist Module}: It explicitly disentangles redundant appearance features in the 3D face parameter space, preserving target lip motion information while effectively removing irrelevant factors such as color and texture, thereby mitigating identity drift and appearance distortion. 2) \textbf{Structure Controller}: It adaptively refines the embedding space based on the dynamic trends of reference lip movements, improving lip-control accuracy and synchronization. 3) \textbf{Noise Sensor}: It derives a Gaussian prior through hypothesis testing and formulaically models flicker and jitter noise patterns in generated videos. The noise patterns are mathematically proven, and the module integrates a spatially adaptive temporal filter to effectively suppress noise, thereby enhancing temporal consistency in video generation. Experimental results on the CREMA~\cite{cao2014crema} and HDTF~\cite{zhang2021flow} datasets demonstrate that the ``SD + IP-Adapter" backbone and these modules collectively improve both visual quality and lip synchronization, offering a novel fine-tuning-free framework for talking face generation.

In summary, our main contributions are as follows:
\begin{enumerate}
	\item We propose a fine-tuning-free diffusion-based framework, FreeTalkDiff, for talking face generation. It is the first to directly leverage pretrained SD and IP-Adapter models, removing the need for costly fine-tuning on large datasets and greatly reducing resource overhead.
	\item We design a 3DMM-based Structurist. This module disentangles and reassembles the useful lip and appearance features, thereby reinforcing identity preservation.
	\item  We propose an adaptive Structure Controller. By dynamically refining structure embeddings in accordance with the motion trends of reference lip movements, this module significantly improves the precision of lip control, achieving more natural and accurate lip synchronization.
	\item We introduce a Gaussian-prior-based Noise Sensor, which formally models and detects flicker and jitter noise patterns in generated videos, and applies a spatially adaptive temporal filter to effectively enhance temporal consistency and realism.
\end{enumerate}

\section{Related work}
\label{sec:related_work}

\subsection{Talking face generation}

Early talking face generation methods primarily relied on GANs~\cite{goodfellow2014generative,mirza2014conditional,radford2015unsupervised,goodfellow2020generative} and autoregressive networks~\cite{hochreiter1997long,sak2014long,chung2014empirical}, which achieves notable progress in improving lip synchronization but suffers from limited vividness and naturalness. For example, Wav2Lip~\cite{prajwal2020lip} introduces a pretrained SyncNet~\cite{chung2016out} discriminator to enforce audio–visual alignment. MakeItTalk~\cite{zhou2020makelttalk} disentangles audio into content and speaker style, and uses facial landmarks as intermediates. Audio2Head~\cite{wang2021audio2head} enhances realism by predicting 6-DoF head poses and dense motion fields. Later methods such as TalkLip~\cite{wang2023seeing} and SadTalker~\cite{zhang2023sadtalker} enhance articulation accuracy and 3D-awareness, but still suffer from blurry frames and coarse lip detail due to limited model capacity.

With the rise of diffusion models~\cite{ho2020denoising,song2020denoising,rombach2022high,blattmann2023stable}, researchers have explored their potential in talking face generation. For instance, AniPortrait~\cite{wei2024aniportrait} combines 3D mesh prediction with diffusion rendering. Loopy~\cite{jiang2024loopy} incorporates an inter- and intra-clip temporal module to capture long-term motion. MuseTalk~\cite{zhang2024musetalk} introduces selective information sampling and adaptive audio modulation to improve audio–visual alignment. More recent models, LatentSync~\cite{li2024latentsync}, EchoMimic~\cite{chen2025echomimic}, and the Hallo family~\cite{xu2024hallo,cui2025hallo}, further advance temporal consistency, multimodal conditioning, and resolution. Sonic~\cite{ji2025sonic} further extend intra-clip audio perception to the global inter-clip level via motion-decoupled control and time-aware position shift fusion. Despite their advances, diffusion-based methods involve fine-tuning billion-scale parameters on large audiovisual datasets, demanding extensive GPU resources, as shown in \cref{tab:diffusion_method_info}. \textbf{This work aims to overcome this challenge by proposing a fine-tuning-free, highly synchronized, and high-fidelity solution for talking face generation.}

In addition, these methods can be divided into one-shot and few-shot paradigms based on the number of ID reference. One-shot methods rely on a single ID reference frame but often struggle with identity preservation, whereas few-shot methods leverage short reference videos to maintain personal dynamics and achieve higher fidelity. \textbf{Accordingly, our work focuses on the few-shot paradigm.}

\subsection{IP-Adapter}

In conditional diffusion-based generation, efficiently incorporating image prompts remains a key challenge. IP-Adapter~\cite{ye2023ip} provides a solution by a CLIP Image Encoder~\cite{radford2021learning} to extract a global image embedding (\ie, \textbf{structure embedding}) and aligning it with the text-conditioned space through a lightweight adapter, enabling flexible control without modifying the SD backbone. Its variant, IP-Adapter-FaceID, further introduces identity embeddings from ArcFace~\cite{deng2019arcface} to enhance identity. These developments provide the technical foundation for using pretrained IP-Adapter to achieve fine-tuning-free talking face generation.

%\subsection{3D Morphable Model}
%
%The core idea of 3D Morphable Model (3DMM)~\cite{huber2016multiresolution,blanz2023morphable} is to use a statistical model to construct low-dimensional linear subspace representations in both the shape and texture spaces of human faces. Specifically, 3DMM performs PCA on a large set of 3D face scans to obtain the mean shape $\bar{P}$, mean texture $\bar{Q}$, and the corresponding basis matrices ${v_i^P}$ and ${v_i^Q}$. The shape $P \in \mathbb{R}^{3n}$ and texture $Q \in \mathbb{R}^{3n}$ of any face can be represented as a linear combination of the mean and several principal component bases:
%\begin{equation}
%	P = \bar{P} + \sum_{i=1}^M \alpha_i \sigma_i^P v_i^P,~~~~
%	Q = \bar{Q} + \sum_{i=1}^M \beta_i \sigma_i^Q v_i^Q,
%\end{equation}
%where $\alpha_i$ and $\beta_i$ are the parameter coefficients for shape and texture, respectively, while $\sigma_i^P$ and $\sigma_i^Q$ denote the corresponding eigenvalues. By adjusting these coefficients in the parameter space, one can generate face instances with different lip shapes or identity characteristics within the low-dimensional subspace. This parametric representation is widely used in talking face generation tasks. The advantage of 3DMM lies in its ability to separate shape, which includes lip-shape information, from texture, which includes appearance information, thereby providing explicit structured conditions for talking face generation.
\section{Motivation}

\begin{figure}[t]
	\centering
	\includegraphics[width=\linewidth]{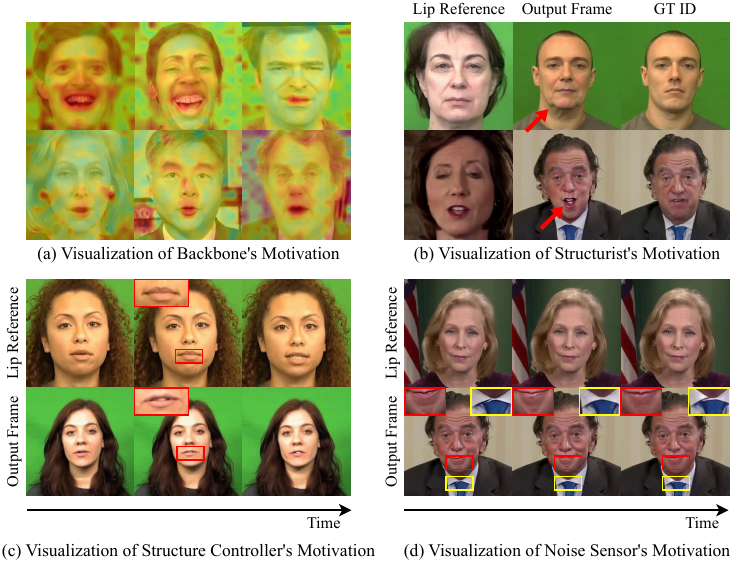}
	\caption{Visualization for motivations.}
	\label{fig:vis_motivation}
\end{figure}

\begin{figure*}[t]
\centering
\includegraphics[width=0.75\linewidth]{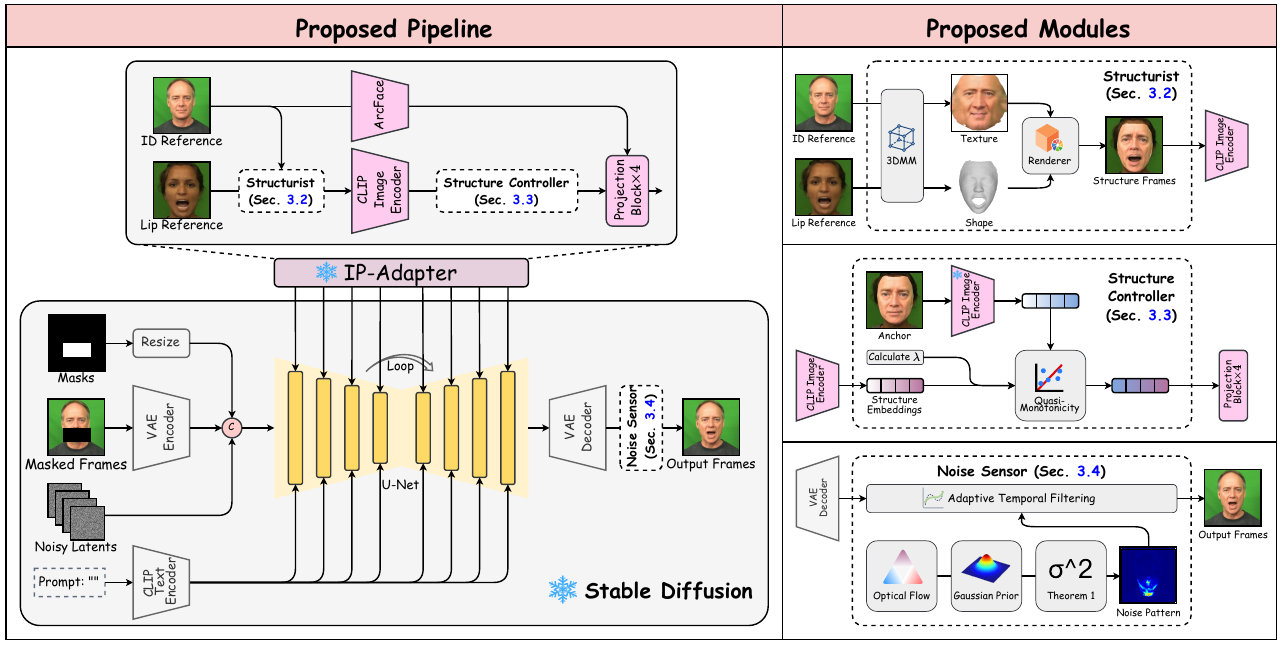}
\caption{The pipeline of the proposed method.}
\label{fig:pipeline}
\end{figure*}

\textbf{The Motivation of the Backbone.} Our goal is to achieve fine-tuning-free talking face generation to eliminate the heavy resource and time costs of training diffusion models. This leads to a key challenge: \textit{How to mine and control the lip-related information from pretrained diffusion models?} 
%\begin{quote}
%	\textit{How to mine and control the lip-related information from pretrained diffusion models?}
%\end{quote}
To explore this question, we conduct an in-depth analysis of pretrained model characteristics. As shown in \cref{fig:vis_motivation}(a), we observed that structure embeddings in IP-Adapter exhibit strong attention to the mouth region, likely because its CLIP Image Encoder emphasizes semantic alignment between image regions and textual concepts, where the mouth often carries key semantic cues in face-related contexts~\cite{radford2021learning,zhao2020hearing}. Motivated by this, we leverage IP-Adapter to inject lip cues from the target lip reference into its partner, SD. Consequently, we adopt the pretrained ``SD + IP-Adapter" backbone for fine-tuning-free talking face generation.

\textbf{The Motivation of the Structurist.} As shown in \cref{fig:vis_motivation}(b), although IP-Adapter can achieve lip-sync control, directly feeding the lip reference introduces redundant texture and color, causing identity drift and appearance distortion. This leads to a key question: \textit{How to remove redundant appearance information while preserving useful lip-related cues?} 
%\begin{quote}
%	\textit{How to remove redundant appearance information while preserving useful lip-related cues?}
%\end{quote}
%To address this, we design the Structurist module, which explicitly separates and recomposes lip-shape and appearance features in the 3D face parameter space. retaining useful appearance and lip-related cues for accurate lip-driven motion and consistent identity appearance.
One possible approach is to explicitly disentangle lip-shape and appearance representations in the 3D face parameter space, allowing the system to retain useful lip-related information while removing interference from extraneous textures, thereby improving both lip-sync accuracy and identity consistency.

\textbf{The Motivation of the Structure Controller.} We further observe that IP-Adapter fail to capture subtle lip motion (\cref{fig:vis_motivation}(c)). This motivates another question: \textit{How can the structure embedding be refined?} 
%\begin{quote}
%	\textit{How to refine structure embeddings using lip references?}
%\end{quote}
%To address this issue, we introduce the Structure Controller, an adaptive mechanism that dynamically adjusts structure embeddings according to the motion trends of reference lips and the mathematical properties of the structure embedding space.
A feasible direction is to develop an adaptive refinement mechanism that adjusts the structure embedding according to the geometric properties of the embedding space, enabling to better track subtle lip dynamics and produce more precise lip synchronization.

\textbf{The Motivation of the Noise Sensor.} Finally, flicker and jitter artifacts often appear in generated mouth regions, as shown in \cref{fig:vis_motivation}(d), disrupting temporal smoothness. This raises the question: \textit{How to suppress temporal noise to improve the visual stability?} 
%\begin{quote}
%	\textit{How to model and suppress temporal noise to improve the visual stability?}
%\end{quote}
%To solve this, we design a Noise Sensor module that derives a Gaussian prior w.r.t. the movement field and formulaically models flicker and jitter noise patterns. Based on this, it applies spatially adaptive temporal Gaussian filtering to suppress noise, significantly enhancing temporal stability and visual realism.
A promising direction is to introduce a noise–aware mechanism that models flicker and jitter over local motion variations, enabling spatially adaptive temporal filtering that selectively attenuates instability while preserving genuine lip movements, thereby improving both temporal consistency and overall visual fidelity.
\section{Methodology}

\subsection{Pipeline overview}

As shown in \cref{fig:pipeline}, our pipeline builds on the ``image inpainting SD + IP-Adapter" backbone to support few-shot paradigm. Specifically, the encoded masked frames, random noise latents and scaled masks are first concatenated, and then passed into the denoising U-Net~\cite{ronneberger2015u}. At each layer, we incorporate the FaceID version of the IP-Adapter to additionally inject identity embeddings. Each IP-Adapter layer receives the previous-layer latent, the text prompt, and two image prompts: identity and lip reference frames. The lip reference is first converted to structure frames via Structurist (\cref{sec:structurist}), and then encoded by CLIP Image Encoder and refined by Structure Controller (\cref{sec:structure_controller}) to capture structure embeddings with lip-shape and appearance cues, while the identity reference is processed by ArcFace~\cite{deng2019arcface} to enhance identity consistency. Finally, the fused embeddings are decoded into continuous talking frames via Noise Sensor (\cref{sec:noise_sensor}). \textbf{Notably, both SD and IP-Adapter are large-scale pretrained models, and our modules require no trainable parameters or fine-tuning, allowing direct application to talking face generation.}

\subsection{3DMM-based Structurist}
\label{sec:structurist}

In this section, we propose a 3DMM-based Structurist module to address the issue of redundant conditional information. As illustrated in \cref{fig:pipeline}, we first employ 3DMM~\cite{huber2016multiresolution,blanz2023morphable} to project the lip reference faces into the shape and texture spaces, obtaining low-dimensional parameter representations that describe geometric structure and appearance attributes, respectively. Specifically, 3DMM represents the shape $P \in \mathbb{R}^{3n}$ and texture $Q \in \mathbb{R}^{3n}$ of any face as linear combinations of mean components and principal bases: $P = \bar{P} + \sum_{i=1}^M \alpha_i \sigma_i^P v_i^P, Q = \bar{Q} + \sum_{i=1}^M \beta_i \sigma_i^Q v_i^Q$, 
%\begin{equation}
%	P = \bar{P} + \sum_{i=1}^M \alpha_i \sigma_i^P v_i^P,~~~~
%	Q = \bar{Q} + \sum_{i=1}^M \beta_i \sigma_i^Q v_i^Q,
%\end{equation}
where $\alpha_i$ and $\beta_i$ denote the shape and texture coefficients, respectively, and $\sigma_i^P$, $\sigma_i^Q$ are the corresponding eigenvalues. The shape parameters capture lip-motion dynamics, while the texture parameters preserve color and detailed appearance information. Subsequently, to avoid interference from appearance, we combine the texture parameters of the target identity reference with the shape parameters of the lip reference, constructing a parametric representation that encapsulates the target appearance while retaining the desired lip-motion. Finally, using a renderer~\cite{laine2020modular}, this parametric representation is mapped to the image domain to generate a structure image prompt, referred to as the \textbf{structure frame}. This prompt is then fed into the CLIP Image Encoder of the IP-Adapter to obtain the structure embedding, which serves as the lip-motion condition for SD.

\subsection{Quasi-monotonicity-based adaptive Structure Controller}
\label{sec:structure_controller}

%\begin{figure}[t]
%	\centering
%	\begin{minipage}[!t]{0.23\textwidth}
%		\centering
%		\includegraphics[width=\linewidth]{figs/clip_manifold}
%		\caption{Visualization for Structure Controller.}
%		\label{fig:clip_manifold}
%	\end{minipage}
%	\hfill
%	\begin{minipage}[!t]{0.23\textwidth}
%		\centering
%		\includegraphics[width=\linewidth]{figs/vis_flows_compressed}
%		\caption{Visualization for Noise Sensor.}
%		\label{fig:vis_flows}
%	\end{minipage}
%\end{figure}
\begin{figure*}[t]
	\centering
	\begin{minipage}[!t]{0.23\textwidth}
		\centering
		\includegraphics[width=\linewidth]{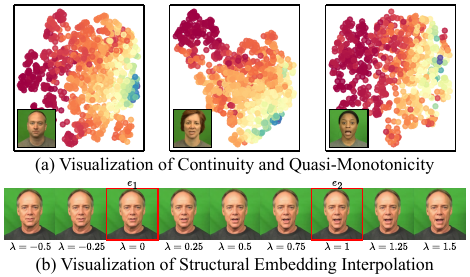}
		\caption{Visualization for Structure Controller.}
		\label{fig:clip_manifold}
	\end{minipage}
	\hfill
	\begin{minipage}[!t]{0.75\textwidth}
		\centering
		\includegraphics[width=\linewidth]{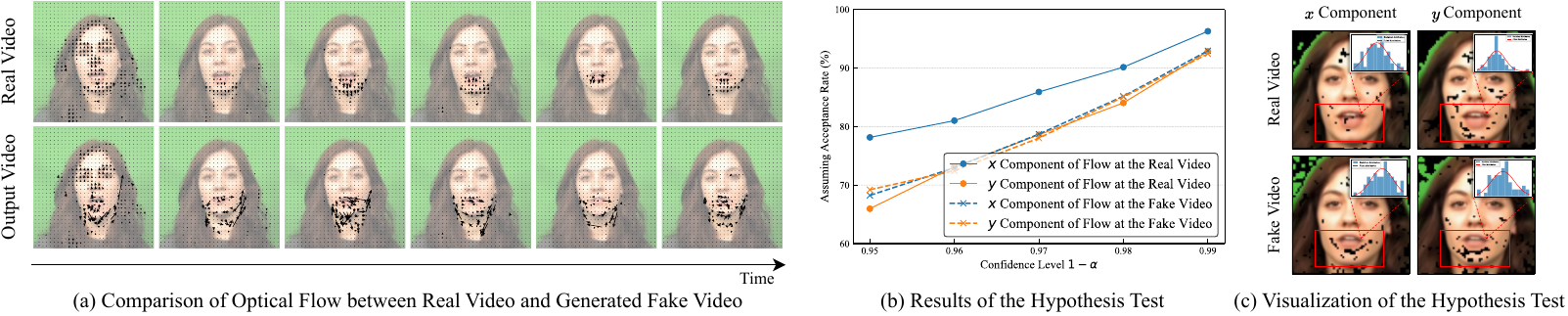}
		\caption{Visualization for Noise Sensor.}
		\label{fig:vis_flows}
	\end{minipage}
\end{figure*}

%\begin{figure}[t]
%	\centering
%	\includegraphics[width=\linewidth]{figs/clip_manifold}
%	\caption{Visualization for Structure Controller.}
%	\label{fig:clip_manifold}
%\end{figure}

To enhance the precision of lip-motion control, we introduce a quasi-monotonicity-based adaptive Structure Controller. We first analyze the mapping from structure embeddings to lip shapes and find two key properties in the structure embedding space $\mathcal{E}$.

\textbf{Continuity and Quasi-Monotonicity.} For each identity reference clip, we extract all structure embeddings and compute the Euclidean distances between upper and lower lip landmarks of generated frames as the corresponding \textit{lip distances}. The high-dimensional embeddings are then reduced to a 2D manifold using Uniform Manifold Approximation and Projection~\cite{mcinnes2018umap} for visualization, and points are color-coded by lip distance—red for smaller lips and blue for larger ones. As shown in \cref{fig:clip_manifold}(a), two observations emerge: 1) \textbf{Continuity}: embeddings of the same identity form dense, smooth clusters, implying that $\mathcal{E}$ is a continuous manifold. 2) \textbf{Quasi-Monotonicity}: color variations follow a clear directional gradient, showing that lip distance changes approximately monotonically along specific embedding directions. Formally, for any $e_1, e_2 \in \mathcal{E}$, the lip distance mapping $f: \mathcal{E} \rightarrow \mathbb{R}^+$ satisfies:
\begin{equation}
	\scalebox{0.61}{$
	\begin{cases}
		f((1 - \lambda) e_1 + \lambda e_2) \leq \min\{f(e_1), f(e_2)\} & \forall~\lambda \in (-\infty, 0] \\
		\min\{f(e_1), f(e_2)\} \leq f((1 - \lambda) e_1 + \lambda e_2) \leq \max\{f(e_1), f(e_2)\} & \forall~\lambda \in (0, 1) \\
		f((1 - \lambda) e_1 + \lambda e_2) \geq \max\{f(e_1), f(e_2)\} & \forall~\lambda \in [1, +\infty)
	\end{cases}.$}
	\label{eq:quasi_monotonicity}
\end{equation}
%\begin{equation}
%	\scalebox{0.68}{$
%		\min\{f(e_1), f(e_2)\} \leq f((1 - \lambda) e_1 + \lambda e_2) \leq \max\{f(e_1), f(e_2)\}, \forall~\lambda \in (0, 1).
%		$}
%\end{equation}
This property implies that lip distance changes approximately monotonically along the interpolation direction, while extrapolation beyond this range yields over-opened/closed lips (\cref{fig:clip_manifold}(b)).

\textbf{Adaptive Adjustment.} Leveraging the above findings, the Structure Controller dynamically adjusts structure embeddings according to the lip-motion trend of renference. Let $S_{\mathrm{anchor}}$ denote the structure frame with the smallest lip opening and $\gamma(\cdot)$ the lip-distance measurement. For the current frame $S_{\mathrm{current}}$, the adjusted embedding is computed as:
\begin{equation}
	e_{\mathrm{current}} = (1 - \lambda) \mathrm{CLIP}(S_{\mathrm{anchor}}) + \lambda \mathrm{CLIP}(S_{\mathrm{current}}),
\end{equation}
%\begin{equation}
%	\lambda = \frac{\gamma(L_{\mathrm{current}})}{\gamma(L_{\mathrm{previous}})},
%\end{equation}
where $\lambda = \frac{\gamma(L_{\mathrm{current}})}{\gamma(L_{\mathrm{previous}})}$, $\mathrm{CLIP}(\cdot)$ denotes the CLIP Image Encoder, whose output corresponds to the structure embedding, and $L_{\mathrm{current}}$ and $L_{\mathrm{previous}}$ represent the current and previous lip reference frames, respectively. When $\gamma(L_{\mathrm{current}}) > \gamma(L_{\mathrm{previous}})$, the reference lips are opening, so $\lambda > 1$, causing $e_{\mathrm{current}}$ to extrapolate in the embedding space toward a more ``open" direction according to \cref{eq:quasi_monotonicity}, generating larger lips; conversely, when $\gamma(L_{\mathrm{current}}) < \gamma(L_{\mathrm{previous}})$, the reference lips are closing, so $\lambda < 1$, pulling $e_{\mathrm{current}}$ back toward the anchor, generating smaller lips. This adaptive adjustment mechanism leverages the quasi-monotonicity of the embedding space, allowing the generated lip shapes to be further corrected along the reference lip-motion trends, thereby compensating for discrepancies between the generated and reference lip shapes.

\subsection{Gaussian-prior-based Noise Sensor}
\label{sec:noise_sensor}

%\begin{figure}[t]
%	\centering
%	\includegraphics[width=\linewidth]{figs/vis_flows}
%	\caption{Visualization for Noise Sensor.}
%	\label{fig:vis_flows}
%\end{figure}

In this section, we propose a Gaussian-prior-based Noise Sensor to alleviate temporal continuity in generated videos.

\textbf{Gaussian Prior and Hypothesis Testing.} As shown in \cref{fig:vis_flows}(a), the optical flow in the lip region of real videos exhibits smooth, continuous patterns, while that in generated videos often appears disordered and oscillatory due to jitter and flicker. 
%To analyze this discrepancy, we model the optical flow vector of pixel $(i, j)$ between consecutive frames as a 2D Gaussian distribution: $\mathbf{V}_{ij} \sim \mathcal{N}(\boldsymbol{\mu}_{ij}, \mathbf{\Sigma}_{ij})$, where $\boldsymbol{\mu}_{ij} = (\mu_{\mathbf{V}_{ij,x}}, \mu_{\mathbf{V}_{ij,y}})^T$ represents the expected optical flow in the $x$ and $y$ directions, and $\Sigma_{ij}$ denotes covariance matrix. 
To further analyze this discrepancy, we establish the following modeling hypothesis from a statistical distribution perspective.
\begin{hypothesis}\label{hyp:gaussian}
	Let the optical flow vector of pixel $(i, j)$ between two consecutive frames be $\mathbf{V}_{ij}$, which can be modeled as a 2D Gaussian distribution: $\mathbf{V}_{ij} \sim \mathcal{N}(\boldsymbol{\mu}_{ij}, \mathbf{\Sigma}_{ij})$, where the mean vector $\boldsymbol{\mu}_{ij} = (\mu_{\mathbf{V}_{ij,x}}, \mu_{\mathbf{V}_{ij,y}})^T$ represents the expected optical flow in the $x$ and $y$ directions, and $\Sigma_{ij}$ is the covariance matrix describing the correlation between the two directions.
\end{hypothesis}
We validate this Gaussian hypothesis using the Shapiro-Wilk test~\cite{shaphiro1965analysis} on lip-region optical flow sequences from both real and generated videos. As shown in \cref{fig:vis_flows}(b), when the confidence level increases to 0.99, 92.4\%–96.3\% of pixels follow a Gaussian distribution. The fitted distributions in \cref{fig:vis_flows}(c) further confirm that most optical flow components (\ie, unmasked pixels) in the lip region can be well modeled by a Gaussian prior, validating its rationality.

\begin{table*}[!t]
	\centering
	\caption{Quantitative comparison with existing talking face generation methods on the CREMA and HDTF datasets. ``{\large $\circ$}" denotes one-shot setting, ``{\large $\bullet$}" denotes few-shot setting, ``{\small $\square$}" indicates non-diffusion-based methods, and ``{\small $\blacksquare$}" indicates diffusion-based methods.}
	\resizebox{\textwidth}{!}{%
		\begin{tabular}{@{}c|c|cccccc|cccccc@{}}
			\toprule
			\multirow{2}{*}{\textbf{Methods}} & \multirow{2}{*}{\textbf{Types}} & \multicolumn{6}{c|}{\textbf{CREMA}} & \multicolumn{6}{c}{\textbf{HDTF}} \\ 
			\cmidrule{3-8} \cmidrule{9-14}
			&  & \textbf{PD}$\downarrow$ & \textbf{CSLD}$\uparrow$ & \multicolumn{1}{c|}{\textbf{PCLD}$\uparrow$} & \textbf{FID}$\downarrow$ & \textbf{LPIPS}$\downarrow$ & \textbf{CPBD}$\uparrow$ & \textbf{PD}$\downarrow$ & \textbf{CSLD}$\uparrow$ & \multicolumn{1}{c|}{\textbf{PCLD}$\uparrow$} & \textbf{FID}$\downarrow$ & \textbf{LPIPS}$\downarrow$ & \textbf{CPBD}$\uparrow$ \\ 
			\midrule
			Wav2Lip(MM'20)~\cite{prajwal2020lip} & {\large $\bullet$}{\small $\square$} & 0.01950 & 0.649 & \multicolumn{1}{c|}{0.235} & 14.9 & 0.056 & 0.134 & 0.01897 & 0.756 & \multicolumn{1}{c|}{0.358} & 12.4 & 0.057 & 0.205 \\
			MakeItTalk(TOG'20)~\cite{zhou2020makelttalk} & {\large $\circ$}{\small $\square$} & 0.02022 & 0.745 & \multicolumn{1}{c|}{0.417} & 74.1 & 0.213 & 0.023 & 0.02443 & 0.752 & \multicolumn{1}{c|}{0.250} & 87.5 & 0.293 & 0.042 \\
			Audio2Head(IJCAI'21)~\cite{wang2021audio2head} & {\large $\circ$}{\small $\square$} & 0.01506 & 0.807 & \multicolumn{1}{c|}{0.531} & 80.2 & 0.313 & 0.058 & 0.01794 & 0.760 & \multicolumn{1}{c|}{0.282} & 92.2 & 0.314 & 0.058 \\
			TalkLip(CVPR'23)~\cite{wang2023seeing} & {\large $\bullet$}{\small $\square$} & 0.01396 & 0.783 & \multicolumn{1}{c|}{0.475} & 15.4 & 0.065 & 0.146 & 0.01736 & 0.805 & \multicolumn{1}{c|}{0.424} & 14.0 & 0.060 & 0.201 \\
			SadTalker(CVPR'23)~\cite{zhang2023sadtalker} & {\large $\circ$}{\small $\square$} & 0.02442 & 0.666 & \multicolumn{1}{c|}{0.337} & 61.6 & 0.294 & 0.108 & 0.01652 & 0.816 & \multicolumn{1}{c|}{0.544} & 35.2 & 0.225 & 0.212 \\
			MuseTalk(arXiv'24)~\cite{zhang2024musetalk} & {\large $\bullet$}{\small $\blacksquare$} & 0.01853 & 0.668 & \multicolumn{1}{c|}{0.313} & 5.2 & 0.037 & 0.143 & 0.01790 & 0.788 & \multicolumn{1}{c|}{0.343} & 2.4 & 0.037 & 0.241 \\
			LatentSync(arXiv'24)~\cite{li2024latentsync} & {\large $\bullet$}{\small $\blacksquare$} & 0.01767 & 0.762 & \multicolumn{1}{c|}{0.429} & 2.2 & 0.032 & 0.189 & 0.01483 & 0.827 & \multicolumn{1}{c|}{0.540} & 1.5 & 0.025 & 0.255 \\
			EchoMimic(AAAI'25)~\cite{chen2025echomimic} & {\large $\circ$}{\small $\blacksquare$} & 0.01925 & 0.665 & \multicolumn{1}{c|}{0.310} & 71.0 & 0.291 & 0.081 & 0.01596 & 0.840 & \multicolumn{1}{c|}{0.524} & 50.1 & 0.476 & 0.145 \\
			Hallo2(ICLR'25)~\cite{cui2025hallo} & {\large $\circ$}{\small $\blacksquare$} & 0.02030 & 0.627 & \multicolumn{1}{c|}{0.121} & 7.9 & 0.175 & 0.131 & 0.01524 & 0.844 & \multicolumn{1}{c|}{0.547} & 3.9 & 0.191 & 0.211 \\ 
			Sonic(CVPR'25)~\cite{ji2025sonic} & {\large $\circ$}{\small $\blacksquare$} & 0.01774 & 0.707 & \multicolumn{1}{c|}{0.408} & 32.9 & 0.228 & 0.104 & 0.01642 & 0.832 & \multicolumn{1}{c|}{0.558} & 23.8 & 0.415 & 0.153 \\
			\midrule
			\textbf{Ours} & {\large $\bullet$}{\small $\blacksquare$} & \textbf{0.00860} & \textbf{0.887} & \multicolumn{1}{c|}{\textbf{0.711}} & \textbf{1.5} & \textbf{0.020} & \textbf{0.218} & \textbf{0.01026} & \textbf{0.883} & \multicolumn{1}{c|}{\textbf{0.718}} & \textbf{0.5} & \textbf{0.023} & \textbf{0.289} \\ 
			\bottomrule
		\end{tabular}%
	}
	\label{tab:quantitative_comparison}
\end{table*}

\textbf{Definition of Noise Pattern.} Under the Gaussian assumption, we characterize the statistical properties of jitter and flicker noise in generated videos as follows:
\begin{theorem}\label{thm:var_noise}
	For pixel $(i,j)$, let $\mathbf{V}_{ij}^{real}$ and $\mathbf{V}_{ij}^{fake}$ denote the optical flow vectors in real and generated videos, respectively. The variance of the generated jitter and flicker noise $\hat{\mathbf{R}}_{ij}$ satisfies:
%	\begin{align}
%		\sigma_{\hat{\mathbf{R}}_{ij,x}}^2 &= \sigma_{\mathbf{V}_{ij,x}^{fake}}^2 - \frac{\mathrm{Cov}^2(\mathbf{V}_{ij,x}^{fake}, \mathbf{V}_{ij,x}^{real})}{\sigma_{\mathbf{V}_{ij,x}^{real}}^2} \geq 0, \\
%		\sigma_{\hat{\mathbf{R}}_{ij,y}}^2 &= \sigma_{\mathbf{V}_{ij,y}^{fake}}^2 - \frac{\mathrm{Cov}^2(\mathbf{V}_{ij,y}^{fake}, \mathbf{V}_{ij,y}^{real})}{\sigma_{\mathbf{V}_{ij,y}^{real}}^2} \geq 0,
%	\end{align}
	\begin{equation}
		\sigma_{\hat{\mathbf{R}}_{ij}}^2 = \sigma_{\mathbf{V}_{ij}^{fake}}^2 - \frac{\mathrm{Cov}^2(\mathbf{V}_{ij}^{fake}, \mathbf{V}_{ij}^{real})}{\sigma_{\mathbf{V}_{ij}^{real}}^2} \geq 0,
	\end{equation}
	where $\sigma_{\mathbf{V}_{ij}}^{real},\sigma_{\mathbf{V}_{ij}}^{fake}$ are the standard deviations of the optical flow vectors for the real and generated videos, respectively. (The proof is detailed in \cref{sec:proof_var_noise} of Supplementary Material.)
\end{theorem}

A smaller noise variance indicates better temporal consistency between generated and real optical flow, while larger variance reflects more severe jitter and flicker. Accordingly, we define the \textbf{noise pattern} as: $D_{ij} = \left\| \sigma_{\hat{\mathbf{R}}_{ij}} \right\|_2 = \sqrt{\sigma_{\hat{\mathbf{R}}_{ij,x}}^2 + \sigma_{\hat{\mathbf{R}}_{ij,y}}^2}$, 
%\begin{equation}
%	D_{ij} = \left\| \sigma_{\hat{\mathbf{R}}_{ij}} \right\|_2
%	= \sqrt{\sigma_{\hat{\mathbf{R}}_{ij,x}}^2 + \sigma_{\hat{\mathbf{R}}_{ij,y}}^2},
%\end{equation}
where higher $D_{ij}$ values correspond to stronger jitter or flicker noise at the corresponding pixel location. 

\textbf{Adaptive Temporal Smoothness.} To suppress flicker and jitter, we design a spatially adaptive temporal filter that adjusts its strength according to the estimated noise intensity. For each pixel $(i,j)$ within a temporal window, the noise pattern $D_{ij}$ serves as the standard deviation of a 1D Gaussian kernel, thus ensuring stronger smoothing in noisier regions. The kernel is %$G_{ij}(k) = \frac{\mathrm{exp}\left(-\frac{k^2}{2 D_{ij}^2}\right)}{\sum \mathrm{exp}\left(-\frac{k^2}{2 D_{ij}^2}\right)}$, 
$G_{ij}(k) = \mathrm{softmax}(-\frac{k^2}{2 D_{ij}^2})$, 
%\begin{equation}
%	G_{ij}(k) = \frac{\mathrm{exp}\left(-\frac{k^2}{2 D_{ij}^2}\right)}{\sum \mathrm{exp}\left(-\frac{k^2}{2 D_{ij}^2}\right)},
%\end{equation}
where $k$ denotes the temporal offset from the current frame. This spatially varying design effectively mitigates flicker and enhances visual stability in generated videos.

\section{Experiments and analysis}

\subsection{Experimental setup}

\textbf{Datasets.} We evaluate our method on two widely used datasets: CREMA~\cite{cao2014crema} (with green-screen backgrounds) and HDTF~\cite{zhang2021flow} (high-definition in-the-wild videos). All clips are standardized to $512 \times 512$ resolution, 25 fps, and 16 kHz audio sampling rate.

\textbf{Baselines.} We compare against ten representative methods covering GAN-based, autoregressive, and recent diffusion-based approaches: Wav2Lip~\cite{prajwal2020lip}, MakeItTalk~\cite{zhou2020makelttalk}, Audio2Head~\cite{wang2021audio2head}, TalkLip~\cite{wang2023seeing}, SadTalker~\cite{zhang2023sadtalker}, MuseTalk~\cite{zhang2024musetalk}, LatentSync~\cite{li2024latentsync}, EchoMimic~\cite{chen2025echomimic}, Hallo2~\cite{cui2025hallo}, and Sonic~\cite{ji2025sonic}.

\textbf{Metrics.} Lip-sync accuracy is evaluated using Procrustes Disparity (PD)~\cite{gower1975generalized}, which removes rigid transformations to align 3D lip landmarks, and provides a robust measure of geometric consistency. To assess the lip dynamic consistency, we use Cosine Similarity (CSLD) and Pearson Correlation (PCLD) of Lip Distance sequences. Visual quality is assessed via FID~\cite{heusel2017gans}, LPIPS~\cite{zhang2018unreasonable}, and CPBD~\cite{bohr2013no}, measuring distribution distance, perceptual similarity, and clarity, respectively.

\textbf{Implementation Details.} Our framework builds upon Image Inpainting SD 1.5~\cite{rombach2022high} and IP-Adapter-FaceID~\cite{ye2023ip}. We adopt the DPM-Solver++ scheduler~\cite{lu2025dpm} with 20 steps for efficient sampling. The Noise Sensor employs a adaptive Gaussian temporal filter with a kernel size of 5. Text prompts are left empty to avoid semantic interference.% All experiments are conducted on a single NVIDIA RTX 4090 GPU.

\subsection{Comparative experiments}

\begin{figure*}[t]
	\centering
	\includegraphics[width=\linewidth]{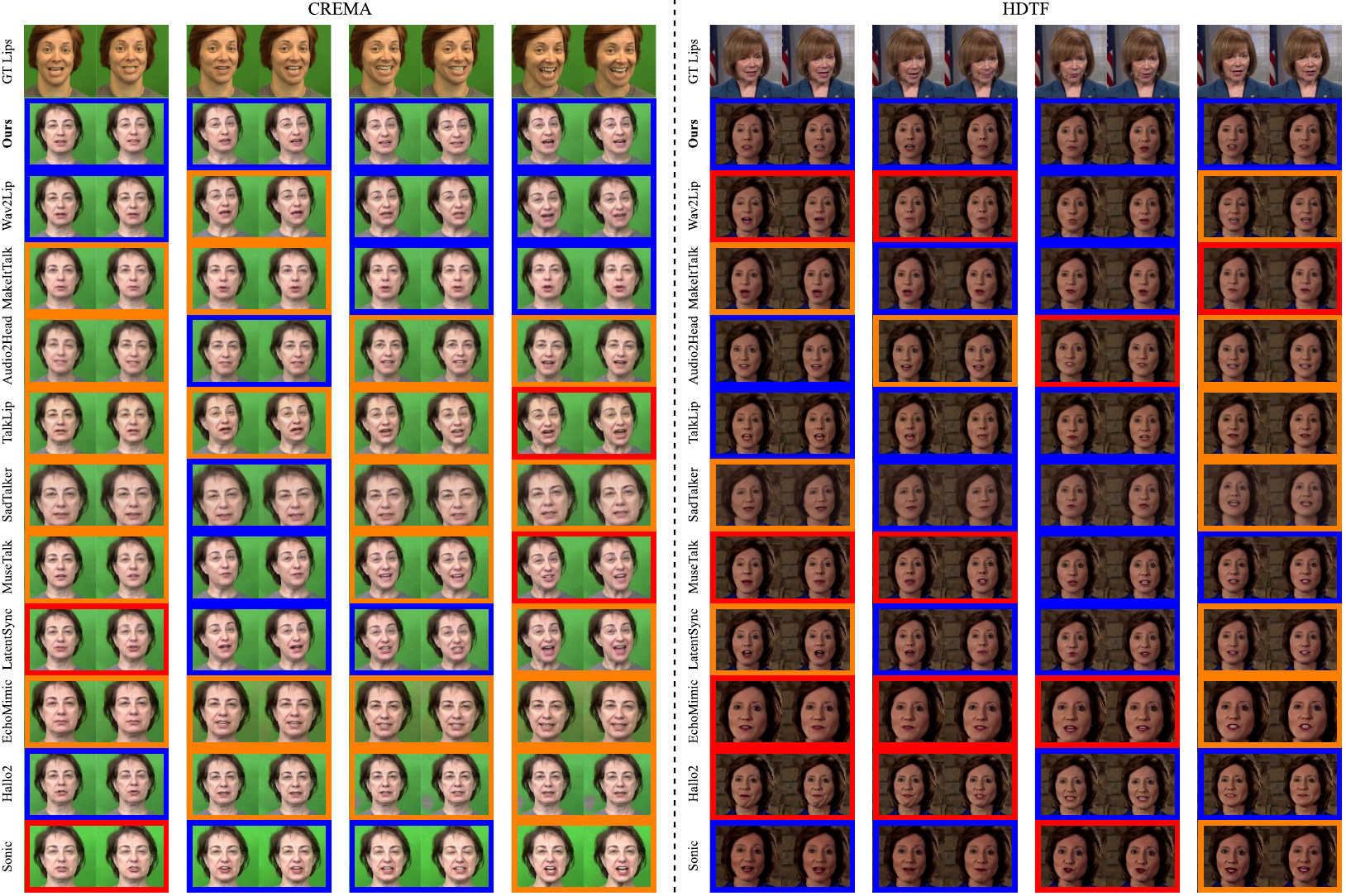}
	\caption{Visualization of generated frames. {\color{blue}Blue}, {\color{orange}orange}, and {\color{red}red} boxes denote correct, weak, and opposite lip-motion trends relative to the ground-truth reference frames, respectively.}
	\label{fig:vis_frames}
\end{figure*}

\begin{figure*}[t]
	\centering
	\includegraphics[width=0.73\linewidth]{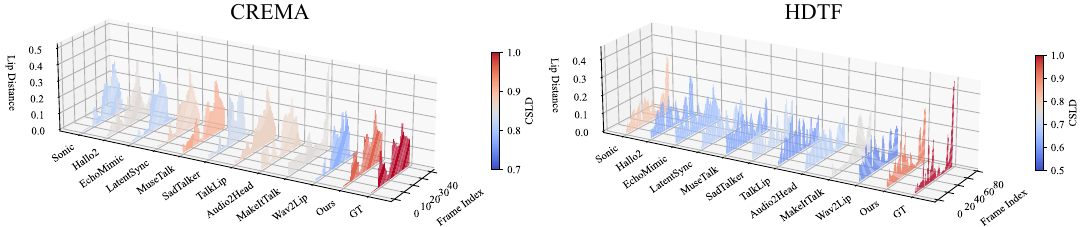}
	\caption{Visualization of lip distance. {\color{red}Redder} indicates higher CSLD, {\color{blue}bluer} indicates lower CSLD.}
	\label{fig:vis_lip_dst}
\end{figure*}

\begin{figure*}[htbp]
	\centering
	\begin{minipage}[!t]{0.23\textwidth}
		\centering
		\includegraphics[width=0.75\textwidth]{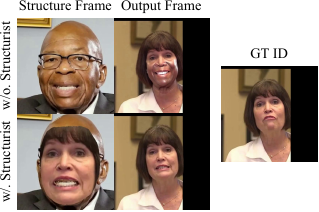} % 示例图片A
		\caption{Visualization ablation results for Structurist.}
		\label{fig:vis_wo_structure}
	\end{minipage}
	\hfill
	\begin{minipage}[!t]{0.23\textwidth}
%		\centering
%		\captionof{table}{Quantitative results of the ablation study for Structure Controller.}
%		\resizebox{\columnwidth}{!}{%
%			\begin{tabular}{@{}c|ccc|ccc@{}}
%				\toprule
%				\multirow{2}{*}{\makecell{\textbf{Structure}\\\textbf{Controller}}} & \multicolumn{3}{c|}{\textbf{CREMA}} & \multicolumn{3}{c}{\textbf{HDTF}} \\
%				\cmidrule{2-7}
%				& \textbf{PD}$\downarrow$ & \textbf{CSLD}$\uparrow$ & \textbf{PCLD}$\uparrow$ & \textbf{PD}$\downarrow$ & \textbf{CSLD}$\uparrow$ & \textbf{PCLD}$\uparrow$ \\
%				\midrule
%				\ding{55} & 0.01035 & 0.858 & 0.650 & 0.01135 & 0.859 & 0.686 \\
%				\ding{51} & \textbf{0.00860} & \textbf{0.887} & \textbf{0.711} & \textbf{0.01026} & \textbf{0.883} & \textbf{0.718} \\
%				\bottomrule
%			\end{tabular}%
%		}
%		\label{tab:ablation_structure_controller}
		\centering
		\includegraphics[width=1.0\textwidth]{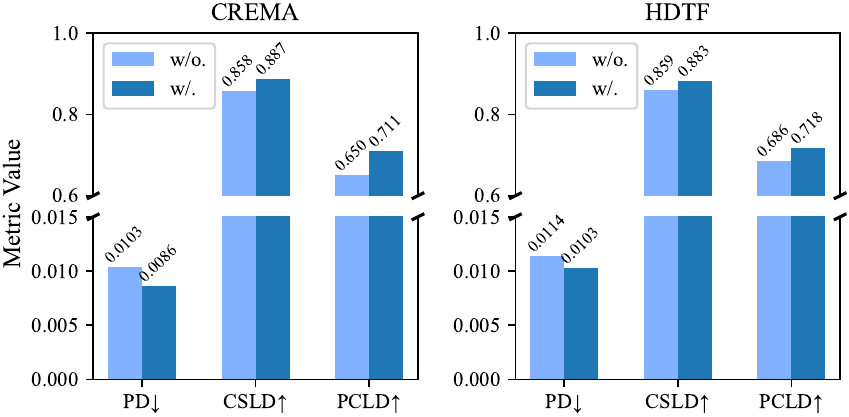} % 示例图片A
		\caption{Quantitative ablation results for Structure Controller.}
		\label{fig:ablation_structure_controller}
	\end{minipage}
	\hfill
	\begin{minipage}[!t]{0.23\textwidth}
		\centering
		\includegraphics[width=\linewidth]{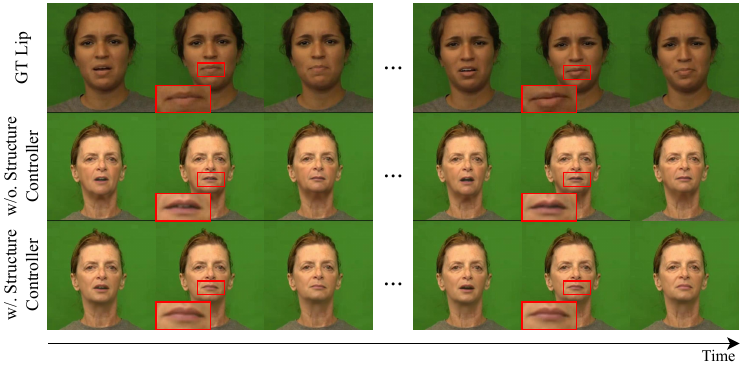}
		\caption{Visualization ablation for Structure Controller.}
		\label{fig:vis_wo_structure_controller}
	\end{minipage}
	\hfill
	\begin{minipage}[!t]{0.23\textwidth}
%		\centering
%		\captionof{table}{Quantitative results of the ablation study for Noise Sensor.}
%		\resizebox{\columnwidth}{!}{
%			\begin{tabular}{@{}c|cc|cc@{}}
%				\toprule
%				\multirow{2}{*}{\textbf{Noise Sensor}} & \multicolumn{2}{c|}{\textbf{CREMA}} & \multicolumn{2}{c}{\textbf{HDTF}} \\ 
%				\cmidrule{2-5} 
%				& \textbf{FVD}$\downarrow$ & \textbf{MNP}$\downarrow$ & \textbf{FVD}$\downarrow$ & \textbf{MNP}$\downarrow$ \\ 
%				\midrule
%				\ding{55} & 349.7 & 0.775 & 537.9 & 1.316 \\
%				\ding{51} & \textbf{141.3} & \textbf{0.120} & \textbf{115.3} & \textbf{0.118} \\
%				\bottomrule
%			\end{tabular}%
%		}
%		\label{tab:ablation_noise_sensor}
		\centering
		\includegraphics[width=\linewidth]{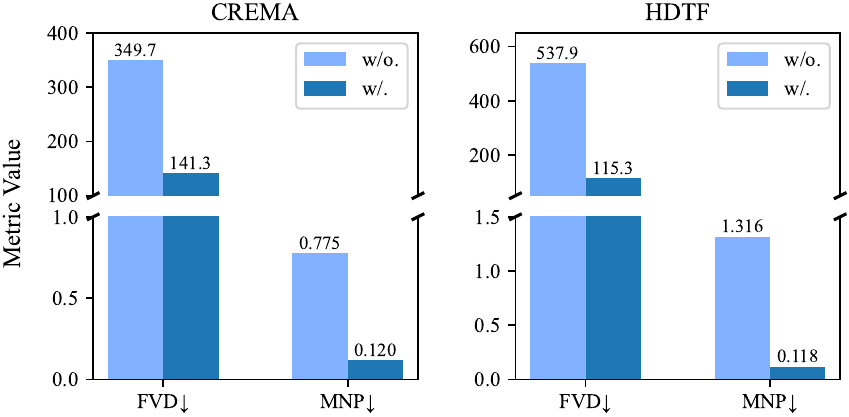}
		\caption{Quantitative ablation results for Noise Sensor.}
		\label{fig:ablation_noise_sensor}
	\end{minipage}
	
	% 两行之间的垂直间距
	%	\vspace{15pt}
	
	\begin{minipage}[!t]{0.23\textwidth}
		\centering
		\includegraphics[width=\linewidth]{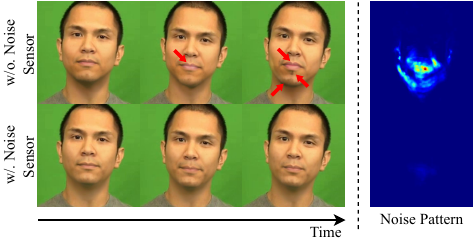}
		\caption{Visualization ablation results for Noise Sensor.}
		\label{fig:vis_wo_noise_sensor}
	\end{minipage}
	\hfill
	\begin{minipage}[!t]{0.23\textwidth}
		\centering
		\includegraphics[width=0.8\linewidth]{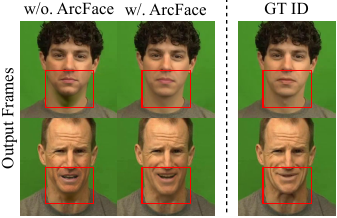}
		\caption{Visualization ablation results for ArcFace.}
		\label{fig:vis_wo_faceid}
	\end{minipage}
	\hfill
	\begin{minipage}[!t]{0.23\textwidth}
		\centering
		\includegraphics[width=0.85\linewidth]{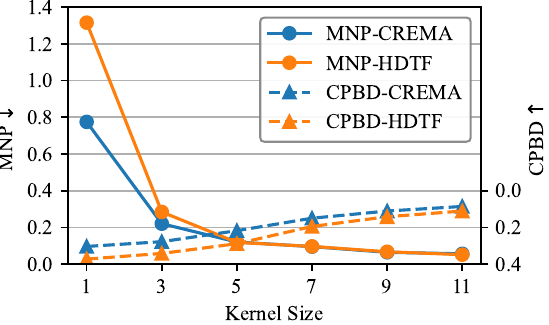}
		\caption{Quantitative ablation results for kernel size.}
		\label{fig:mnp_cpbd_kernel_analysis}
	\end{minipage}
	\hfill
	\begin{minipage}[!t]{0.23\textwidth}
		\centering
		\includegraphics[width=\linewidth]{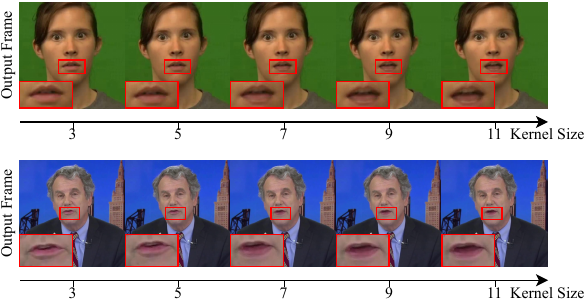}
		\caption{Visualization results with different kernel sizes.}
		\label{fig:vis_diff_ks}
	\end{minipage}
\end{figure*}

\subsubsection{Quantitative comparison}

This section reports quantitative comparisons with ten representative talking face generation methods on the CREMA and HDTF datasets, as shown in \cref{tab:quantitative_comparison}.

In terms of lip-sync performance, our method achieves the lowest PD (CREMA: 0.00860, HDTF: 0.01026), indicating strong geometric consistency between generated and the ground-truth lips. For CSLD and PCLD, which measure temporal consistency and correlation of lip motion, our method also achieves the best results (CREMA: 0.887/0.711, HDTF: 0.883/0.718). These improvements primarily benefit from the adaptive Structure Controller, which dynamically refines structure embeddings for aligned lip motion. Moreover, the fine-tuning-free backbone ensures robust generalization across datasets by leveraging large-scale pretraining on LAION dataset~\cite{schuhmann2022laion}.

Regarding visual quality, our method also surpasses all baselines, achieving FID/LPIPS of 1.5/0.020 on CREMA and 0.5/0.023 on HDTF. Meanwhile, our method achieves the highest CPBD values for image clarity (CREMA: 0.218, HDTF: 0.289), indicating that generated images are closer to real videos in terms of detail. This improvement results from the Structurist's disentanglement mechanism that maintains appearance, and the Noise Sensor, which effectively suppresses flicker and jitter. Notably, few-shot methods generally produce higher visual quality than one-shot methods due to the preservation of the unmasked area. Our method further extends this advantage by enhancing the realism of the mouth area.

%In summary, by integrating structural disentanglement, adaptive control, and temporal noise suppression into a fine-tuning-free diffusion backbone, our method achieves SOTA performance in both lip-sync accuracy and visual realism.

\subsubsection{Qualitative comparison}

\textbf{Visualization of Generated Frames.} To intuitively evaluate the lip-sync performance of our method, we compare its generated frames with those from multiple baseline methods on the CREMA and HDTF datasets. Specifically, we display two consecutive generated frames alongside their corresponding ground-truth lip reference frames, as shown in \cref{fig:vis_frames}. The results show that subtle or opposite lip movement trends are common among baseline methods, reflecting their limitations in lip-motion-driven modeling. In contrast, our method produces lip movements consistent with the ground-truth reference frames, avoiding incorrect directions or lack of motion, thereby achieving significant advantages in lip-sync accuracy.

\textbf{Visualization of Lip Distance.} To provide a more intuitive visualization of our method's effectiveness in lip-sync, we compare the lip distance sequences generated by our method with those of multiple baseline methods. Specifically, we select examples in \cref{fig:vis_frames} and plot the ground-truth lip distance sequences along with the sequences generated by each method. As shown in \cref{fig:vis_lip_dst}, our method's lip distance sequences closely match the ground-truth, accurately capturing the amplitude and rhythm of lip opening and closing, thereby achieving the highest CSLD scores among all methods. In contrast, baseline methods exhibit noticeable deviations in lip opening and closing rhythms, resulting in varying degrees of distortion in lip movements.

%These results indicate that our method generates visually more realistic lip movements, effectively improving lip-sync accuracy.

\subsection{Ablation study}

%\begin{figure}[t]
%	\centering
%	\includegraphics[width=\linewidth]{figs/vis_wo_structure}
%	\caption{Visualization results of the ablation study on the Structurist module. {\color{red}Red} arrows indicate that when the Structurist is not used (w/o. Structurist), texture and color information from the structure frame leak into the output frames, causing identity drift or appearance distortion.}
%	\label{fig:vis_wo_structure}
%\end{figure}

\textbf{Effectiveness of Structurist.} \cref{fig:vis_wo_structure} shows that without the Structurist, redundant texture and color from the structure frame leak into the output, causing identity drift and appearance distortion. Incorporating Structurist effectively suppresses such leakage, yielding results consistent with the GT ID. Additionally, since texture information incorporates expression information, Structurist is also capable of effectively preserving the mouth expressions of GT ID. These results validate the role of Structurist in effectively disentangling lip motion information from appearance features.

%\begin{table}[t]
%	\centering
%	\resizebox{\columnwidth}{!}{%
%		\begin{tabular}{c|ccc|ccc}
%			\toprule
%			\multirow{2}{*}{\textbf{Structure Controller}} & \multicolumn{3}{c|}{\textbf{CREMA}} & \multicolumn{3}{c}{\textbf{HDTF}} \\
%			\cmidrule{2-7}
%			& \textbf{PD}$\downarrow$ & \textbf{CSLD}$\uparrow$ & \textbf{PCLD}$\uparrow$ & \textbf{PD}$\downarrow$ & \textbf{CSLD}$\uparrow$ & \textbf{PCLD}$\uparrow$ \\
%			\midrule
%			\ding{55} & 0.01035 & 0.858 & 0.650 & 0.01135 & 0.859 & 0.686 \\
%			\ding{51} & \textbf{0.00860} & \textbf{0.887} & \textbf{0.711} & \textbf{0.01026} & \textbf{0.883} & \textbf{0.718} \\
%			\bottomrule
%		\end{tabular}%
%	}
%	\caption{Quantitative results of the ablation study on the Structure Controller. ``\ding{55}" indicates without Structure Controller, and ``\ding{51}" indicates with Structure Controller. ``$\downarrow$" means lower is better, while ``$\uparrow$" means higher is better. The \textbf{bold} data indicates the best results.}
%	\label{tab:ablation_structure_controller}
%\end{table}

%\begin{figure}[t]
%	\centering
%	\includegraphics[width=\linewidth]{figs/vis_wo_structure_controller}
%	\caption{Visualization results of the ablation study on the Structure Controller.}
%	\label{fig:vis_wo_structure_controller}
%\end{figure}

\textbf{Effectiveness of Structure Controller.} As shown in \cref{fig:ablation_structure_controller}, removing Structure Controller (w/o.) leads to larger deviations between the generated and reference lip shapes. Introducing it (w/.) significantly reduces PD and improves CSLD and PCLD, indicating better alignment in both amplitude and dynamics. The visualization in \cref{fig:vis_wo_structure_controller} further confirms that the controller enables the generated lips to more accurately follow the reference (GT Lip), demonstrating its importance in adaptively correcting lip motion trends and enhancing lip-sync accuracy.

%\begin{table}[t]
%	\centering
%	%	\resizebox{\columnwidth}{!}{%
%		\begin{tabular}{c|cc|cc}
%			\toprule
%			\multirow{2}{*}{\textbf{Noise Sensor}} & \multicolumn{2}{c|}{\textbf{CREMA}} & \multicolumn{2}{c}{\textbf{HDTF}} \\ 
%			\cmidrule{2-5} 
%			& \textbf{FVD}$\downarrow$ & \textbf{MNP}$\downarrow$ & \textbf{FVD}$\downarrow$ & \textbf{MNP}$\downarrow$ \\ 
%			\midrule
%			\ding{55} & 349.7 & 0.775 & 537.9 & 1.316 \\
%			\ding{51} & \textbf{141.3} & \textbf{0.120} & \textbf{115.3} & \textbf{0.118} \\
%			\bottomrule
%		\end{tabular}%
%		%	}
%	\caption{Quantitative results of the ablation study on the Noise Sensor. The metric \textbf{MNP} denotes the Mean of Noise Pattern. ``\ding{55}" indicates without Structure Controller, and ``\ding{51}" indicates with Structure Controller. ``$\downarrow$" means lower is better. The \textbf{bold} data indicates the best results.}
%	\label{tab:ablation_noise_sensor}
%\end{table}

%\begin{figure}[t]
%	\centering
%	\includegraphics[width=\linewidth]{figs/vis_wo_noise_sensor}
%	\caption{Visualization results of the ablation study on the Noise Sensor. The {\color{red}red} arrows highlight the irregular variations in color and texture around the mouth region when the Noise Sensor is not used (w/o.), indicating the presence of jittering and flickering artifacts.}
%	\label{fig:vis_wo_noise_sensor}
%\end{figure}

\textbf{Effectiveness of Noise Sensor.} As shown in \cref{fig:ablation_noise_sensor}, removing Noise Sensor (w/o.) leads to notably higher FVD and MNP (\ie, Mean of Noise Pattern) values, indicating degraded temporal smoothness. In contrast, incorporating it (w/.) significantly reduces both metrics, confirming its role in enhancing temporal coherence. \cref{fig:vis_wo_noise_sensor} further shows that the Noise Sensor effectively suppresses mouth-region flickering and jittering (red arrows), yielding smoother and more stable video dynamics.

%\begin{figure}[t]
%	\centering
%	\includegraphics[width=0.8\linewidth]{figs/vis_wo_faceid}
%	\caption{Visualization results of the ablation study on ArcFace.}
%	\label{fig:vis_wo_faceid}
%\end{figure}

\textbf{Effectiveness of ArcFace.} Without ArcFace, the IP-Adapter degenerates from the FaceID version to the base version. As shown in \cref{fig:vis_wo_faceid}, the generated mouth region exhibits noticeable color and jawline discrepancies compared to the ground-truth identity face, resulting in a loss of identity information. In contrast, with ArcFace, the generated mouth region closely matches the real face. Visualization results demonstrate that incorporating ArcFace enhances identity preservation.

%\begin{figure}[t]
%	\centering
%	\includegraphics[width=0.8\linewidth]{figs/mnp_cpbd_kernel_analysis}
%	\caption{Quantitative results of the ablation study on temporal Gaussian smoothing with different kernel sizes. A kernel size of 1 indicates no smoothing is applied.}
%	\label{fig:mnp_cpbd_kernel_analysis}
%\end{figure}

%\begin{figure}[t]
%	\centering
%	\includegraphics[width=\linewidth]{figs/vis_diff_ks}
%	\caption{Visualization results of temporal Gaussian smoothing with different kernel sizes.}
%	\label{fig:vis_diff_ks}
%\end{figure}

\textbf{Comparison between Different Kernel Sizes.} As shown in \cref{fig:mnp_cpbd_kernel_analysis}, larger kernels effectively reduce flicker noise (lower MNP) but cause motion blur (lower CPBD). Visual results in \cref{fig:vis_diff_ks} confirm that excessive kernel sizes lead to mouth-region blurring and ghosting. Overall, a kernel size of 5 offers the best trade-off between temporal smoothness and visual clarity.
\section{Conclusion and future work}

This paper introduces a fine-tuning-free diffusion-based framework, FreeTalkDiff, for talking face generation, integrating the Structurist, Structure Controller, and Noise Sensor modules to enhance identity consistency, lip synchronization, and temporal smoothness. Experimental results have demonstrated superior lip-sync accuracy and visual quality on multiple datasets, validating the potential of pretrained diffusion models for high-fidelity talking face synthesis without fine-tuning. 
In future work, we plan to 1) explore multimodal fusion strategies that jointly model semantic, emotional, and expressive cues to enhance naturalness and expressiveness; and 2) integrate large language models with visual diffusion models to develop end-to-end parameter-frozen audiovisual generation systems, offering a new framework for multimodal digital human synthesis.
\section{Acknowledgments}

This work is supported by the Natural Science Foundation for Excellent Young Scholars of Henan Province (Grants No. 252300421233), the National Natural Science Foundation of China (Grants No. U23A20305, 62472229, U24B20179), the Innovation Scientists and Technicians Troop Construction Projects of Henan Province (Grants No. 254000510007), and the Research Project of Quan Cheng Laboratory (Grants No. QCL20250203).

{
	\small
	\bibliographystyle{ieeenat_fullname}
	\bibliography{main}
}
\clearpage
\setcounter{page}{1}
\maketitlesupplementary
\setcounter{section}{0} % 章节从1开始
\setcounter{figure}{0}   % 图片从1开始
\setcounter{table}{0}    % 表格从1开始
\setcounter{equation}{0} % 公式从1开始

\section{The proof of \cref{thm:var_noise}}
\label{sec:proof_var_noise}

This section provides a detailed proof of \cref{thm:var_noise}, aiming to rigorously justify the theoretical claims presented in the main paper. All derivations and intermediate steps are included to ensure completeness and clarity.

Taking the $x$-direction component of the random variable $\mathbf{V}_{ij}$ as an example, we define it as a new random variable $X$. The optical flow component in the $x$ direction of the real video is denoted by $X^{real}$, that of the generated video by $X^{fake}$, the ideal flicker-free and jitter-free version by $Z$, and the noise that causes flicker and jitter by $R$. Therefore, according to the Gaussian prior, these random variables satisfy
\begin{equation}
	X^{real} \sim \mathcal{N}(\mu_{X^{real}}, \sigma_{X^{real}}^2),
\end{equation}
\begin{equation}
	X^{fake} = Z + R \sim \mathcal{N}(\mu_{X^{fake}}, \sigma_{X^{fake}}^2),
\end{equation}

Since the distributions of $Z$ and $R$ are unknown, we aim to approximate the ideal flicker-free motion pattern $Z$ using a Gaussian random variable $\hat{Z}$ according to the Gaussian prior. Assuming that the desired motion pattern should be consistent with the real one, $\hat{Z}$ can be derived from $X^{real}$ through a linear transformation with parameters $\alpha$ and $\beta$:
\begin{equation}
	\hat{Z} = \alpha X^{real} + \beta \sim \mathcal{N}(\alpha \mu_{X^{real}} + \beta, \alpha^2 \sigma_{X^{real}}^2).
\end{equation}

We derive the optimal estimation of $\hat{Z}$ through the following lemma.
\begin{lemma}\label{lem:z_estimation}
	To minimize the estimation error $\mathbb{E}[(Z - \hat{Z})^2]$, the optimal parameters $\alpha$ and $\beta$ are given by
	\begin{equation}
		\begin{cases}
			\alpha^* = \frac{\mathrm{Cov}(Z, X^{real})}{\sigma_{X^{real}}^2} \\
			\beta^* = \mu_Z - \frac{\mathrm{Cov}(Z, X^{real})}{\sigma_{X^{real}}^2} \mu_{X^{real}}
		\end{cases}.
	\end{equation}
	Consequently, the optimal estimation $\hat{Z}^*$ can be expressed as
	\begin{equation}
		\hat{Z}^* = \frac{\mathrm{Cov}(Z, X^{real})}{\sigma_{X^{real}}^2} (X^{real} - \mu_{X^{real}}) + \mu_Z.
	\end{equation}
\end{lemma}
\begin{proof}
	To derive the optimal estimation of $\hat{Z}$, we define the estimation error as
	\begin{equation}
		J(\alpha, \beta) = \mathbb{E}[(Z - \hat{Z})^2] = \mathbb{E}[(Z - (\alpha X^{real} + \beta))^2].
	\end{equation}
	To minimize the estimation error, the first-order derivatives with respect to $\alpha$ and $\beta$ must vanish:
	\begin{equation}
		\begin{cases}
			\frac{\partial J}{\partial \alpha} = -2 \mathbb{E}[(Z - (\alpha X^{real} + \beta)) X^{real}] = 0, \\
			\frac{\partial J}{\partial \beta} = -2 \mathbb{E}[Z - (\alpha X^{real} + \beta)] = 0.
		\end{cases}
	\end{equation}
	Solving the above equations yields the optimal parameters:
	\begin{equation}
		\begin{cases}
			\alpha^* = \frac{\mathrm{Cov}(Z, X^{real})}{\sigma_{X^{real}}^2}, \\
			\beta^* = \mu_Z - \frac{\mathrm{Cov}(Z, X^{real})}{\sigma_{X^{real}}^2} \mu_{X^{real}}.
		\end{cases}
	\end{equation}
	Substituting these parameters into $\hat{Z} = \alpha X^{real} + \beta$ gives the optimal estimation:
	\begin{equation}
		\hat{Z}^* = \frac{\mathrm{Cov}(Z, X^{real})}{\sigma_{X^{real}}^2} (X^{real} - \mu_{X^{real}}) + \mu_Z.
	\end{equation}
	This result indicates that the optimal estimation $\hat{Z}^*$ can be interpreted as a linear regression prediction of the real optical flow distribution $X^{real}$, where the regression coefficient is governed by the covariance relationship between the ideal motion pattern $Z$ and the real one $X^{real}$.
\end{proof}

Based on the optimal estimation $\hat{Z}^*$ derived in \cref{lem:z_estimation}, the estimated noise component can be expressed as
\begin{equation}
	\hat{R} = X^{fake} - \hat{Z}^*.
\end{equation}
Assuming that flicker and jitter are induced by the inherent diffusion noise, \ie, $R$ is statistically independent of the other random variables, the variance of the estimated noise can be written as
\begin{align}
	\sigma_{\hat{R}}^2 &= \mathrm{Var}(X^{fake} - \hat{Z}^*) \\
	&= \sigma_{X^{fake}}^2 + \sigma_{\hat{Z}^*}^2 - 2 \mathrm{Cov}(X^{fake}, \hat{Z}^*) \\
	&= \sigma_{X^{fake}}^2 + \frac{\mathrm{Cov}^2(Z, X^{real})}{\sigma_{X^{real}}^2} \\
	&- 2 \frac{\mathrm{Cov}(Z, X^{real})}{\sigma_{X^{real}}^2} \mathrm{Cov}(X^{fake}, X^{real}) \\
%	&= \sigma_{X^{fake}}^2 + \frac{(\mathrm{Cov}(X^{fake}, X^{real}) - \mathrm{Cov}(R, X^{real}))^2}{\sigma_{X^{real}}^2} \\
	&=\!\sigma_{X^{fake}}^2\!+\!\frac{\!(\!\mathrm{Cov}\!(\!X^{fake},\!X^{real}\!)\!-\!\mathrm{Cov}\!(\!R,\!X^{real}\!)\!)^2}{\sigma_{X^{real}}^2}\!\\
%	&- 2 \frac{\mathrm{Cov}(X^{fake}, X^{real}) - \mathrm{Cov}(R, X^{real})}{\sigma_{X^{real}}^2} \mathrm{Cov}(X^{fake}, X^{real}) \\
	&-\!2\!\frac{\mathrm{Cov}\!(\!X^{fake},\!X^{real}\!)\!-\!\mathrm{Cov}\!(\!R,\!X^{real}\!)\!}{\sigma_{X^{real}}^2} \mathrm{Cov}\!(\!X^{fake},\!X^{real}\!)\!\\
	&= \sigma_{X^{fake}}^2 - \frac{\mathrm{Cov}^2(X^{fake}, X^{real})}{\sigma_{X^{real}}^2}.
\end{align}

\begin{figure*}[!t]
\centering
\includegraphics[width=\linewidth]{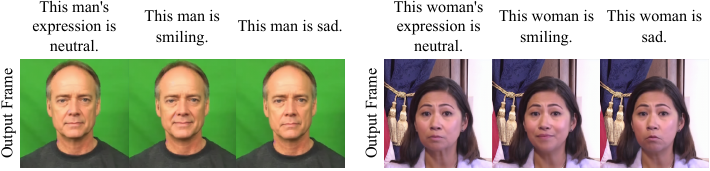}
\caption{Visualization of mouth expression controlled by the text prompt.}
\label{fig:vis_prompt}
\end{figure*}

\begin{figure*}[!t]
\centering
\includegraphics[width=\linewidth]{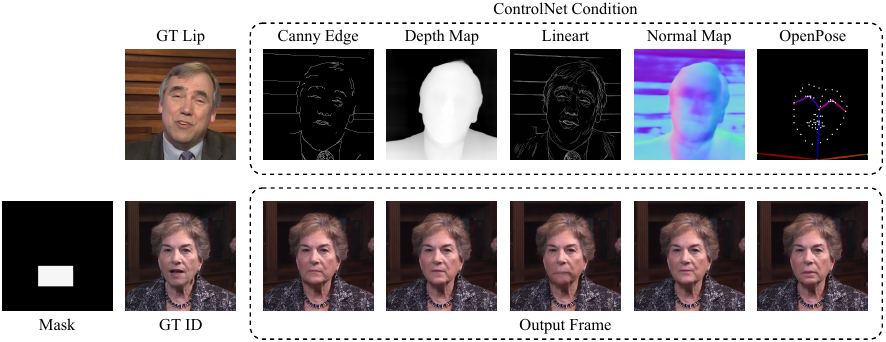}
\caption{Visualization of lip control using different ControlNet conditions.}
\label{fig:vis_controlnet}
\end{figure*}

The result shows that the noise variance depends on the covariance between the generated and real optical flows. The non-negativity of this variance is guaranteed by the Cauchy–Schwarz inequality:
\begin{align}
	\sigma_{\hat{R}}^2 &= \sigma_{X^{fake}}^2 - \frac{\mathrm{Cov}^2(X^{fake}, X^{real})}{\sigma_{X^{real}}^2} \\
	&\geq \sigma_{X^{fake}}^2 - \frac{\sigma_{X^{fake}}^2 \sigma_{X^{real}}^2}{\sigma_{X^{real}}^2} \\
	&= 0.
\end{align}

Similarly, the same result holds for the $y$-direction component of the random variable $\mathbf{V}_{ij}$.

\section{Discussion} 

In this section, we discuss several key aspects of our work. We first explore the controllability of mouth expressions under few-shot settings, then examine the potential development of AnimateDiff and IP-Adapter communities as a backbone for fine-tuning-free talking face generation, and finally analyze the effectiveness of pretrained ControlNet for lip control.

\textbf{Controllable Mouth Expression.} In the few-shot setting, the generated mouth region typically retains the original expression. As demonstrated in the ablation study of the Structurist, the mouth expression primarily originates from the texture information of the structural frame and effectively preserves the expression features of the identity reference. As a functional extension, we further explore the ability to control mouth expressions via text prompts. As shown in \cref{fig:vis_prompt}, when expression-related semantics are injected into the generated frames through the text prompts, the mouth region responds accordingly, demonstrating a certain level of controllability. However, such text-prompt-based expression modification may sometimes cause inconsistencies with other facial regions, leading to a slightly unnatural appearance.

\begin{figure*}[!t]
	\centering
	\includegraphics[width=\linewidth]{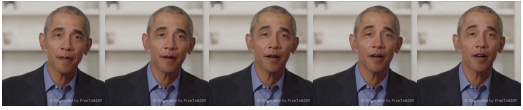}
	\caption{Visualization of the visible watermark.}
	\label{fig:vis_watermark}
\end{figure*}

\textbf{Development of AnimateDiff and IP-Adapter Communities.} AnimateDiff~\cite{guo2024animatediff} extends SD 1.5 by introducing pluggable motion modules and motion LoRA, enabling coherent video generation. This design does not compromise the compatibility with IP-Adapter, which also operates as a plugin within the SD 1.5 framework. Consequently, AnimateDiff can be paired with IP-Adapter to form a \textit{seemingly} feasible fine-tuning-free backbone for talking face generation. However, the current IP-Adapter only provides a shared lip reference for the entire input video clip, rather than assigning distinct lip features to each frame. This limitation causes the ``AnimateDiff + IP-Adapter" backbone to produce static lip shape across frames, preventing it from achieving fine-tuning-free talking face generation. Looking ahead, with the continued evolution of the AnimateDiff and IP-Adapter communities and the increasing modularity and openness of these frameworks, we believe this combined backbone holds great potential to advance fine-tuning-free talking face generation.

\textbf{Analysis of Pretrained ControlNet for Lip Control.} To investigate whether pretrained ControlNet~\cite{zhang2023adding} can effectively control lip generation, we conduct conditional generation experiments using various types of ControlNet inputs, including Canny edge, depth map, lineart, normal map, and OpenPose. As shown in \cref{fig:vis_controlnet}, the generated results reveal that pretrained ControlNet fails to accurately control mouth shapes. This limitation may stem from two main factors: 1) many ControlNet conditions (\eg, Canny edge, depth map, and normal map) tend to destroy lip-related details, and 2) the internal representations of ControlNet are generally not designed to capture fine-grained lip cues—although OpenPose provides mouth landmarks, its pretrained model focuses more on global posture rather than localized mouth dynamics. These findings indicate that pretrained ControlNet lacks explicit sensitivity to lip structures, making it unsuitable for precise lip control. %We hope that future developments in the ControlNet community will introduce lip-aware conditioning mechanisms to better support fine-tuning-free talking face generation.

\section{Ethical considerations}

we recognize that realistic talking face generation may raise ethical concerns regarding potential misuse, such as creating deceptive or malicious deepfake content. To mitigate such risks, all generated videos in our study can be clearly marked as synthetic (\cref{fig:vis_watermark}), ensuring transparent presentation of results. We strongly advocate that any future application of this technology should comply with ethical standards and obtain informed consent when involving identifiable individuals.

\end{document}